\documentclass[onecolumn, a4paper, draftcls, twoside, 11pt]{IEEEtran}
\ifCLASSOPTIONcompsoc
\else
\fi
\ifCLASSINFOpdf
\else
\fi

\usepackage{epsfig}
\usepackage{graphicx}
\usepackage{amsmath}
\usepackage{amssymb}
\usepackage{array}
\usepackage{multirow}
\usepackage{algorithm,algorithmic}
\usepackage{subfigure}

\newtheorem{theorem}{Theorem}
\newtheorem{lemma}[theorem]{Lemma}
\newtheorem{corollary}[theorem]{Corollary}

\newenvironment{proof}[1][Proof]{\begin{trivlist}
\item[\hskip \labelsep {\bfseries #1}]}{\end{trivlist}}

\begin{document}
%
\title{Fixed-Rank Representation for Unsupervised Visual Learning}
%
%
%
%

\author{Risheng~Liu,
        Zhouchen~Lin,
        Fernando~De~la~Torre,
        and~Zhixun~Su,
\IEEEcompsocitemizethanks{\IEEEcompsocthanksitem R. Liu is with School of Mathematical Sciences, Dalian University of Technology, Dalian, P.R. China.
This work is done when R. Liu is visiting the Robotics Institute of Carnegie Mellon University.
E-mail: rsliu0705@gmail.com.
\IEEEcompsocthanksitem Z. Lin is with Microsoft Research Asia and Key Lab. of Machine Perception (MOE), Peking University, P.R. China.
\IEEEcompsocthanksitem F. De~la~Torre is with Robotics Institute, Carnegie Mellon University, USA.
\IEEEcompsocthanksitem Z. Su is with School of Mathematical Sciences, Dalian University of Technology, Dalian, P.R. China.}
\thanks{}}

%
%

\markboth{Technical Report}%
{Liu \MakeLowercase{\textit{et al.}}: Fixed-Rank Representation for Unsupervised Visual Learning}
%


\IEEEcompsoctitleabstractindextext{%
\begin{abstract}
Subspace clustering and feature extraction are two of the most commonly used unsupervised learning
techniques in computer vision and pattern recognition. State-of-the-art techniques for subspace
clustering make use of recent advances in sparsity and rank minimization. However, existing techniques
are computationally expensive and may result in degenerate solutions that degrade clustering performance
in the case of insufficient data sampling. To partially solve these problems, and inspired by existing work on matrix factorization, this paper proposes fixed-rank representation (FRR) as a unified framework for unsupervised visual learning. FRR is able to reveal the structure of multiple subspaces in closed-form when the data is noiseless. Furthermore, we prove that under some suitable conditions, even with insufficient observations, FRR can still reveal the true subspace memberships.  To achieve robustness to outliers and noise, a sparse regularizer is introduced into the FRR framework. Beyond subspace clustering, FRR can be used for unsupervised feature extraction. As a non-trivial byproduct, a fast numerical solver is developed for FRR. Experimental results on both synthetic data and real applications validate our theoretical analysis and demonstrate the benefits of FRR for unsupervised visual learning.
\end{abstract}

}

\maketitle

\IEEEdisplaynotcompsoctitleabstractindextext

%
\IEEEpeerreviewmaketitle
\section{Introduction}
Clustering and embedding are two of the most important techniques for visual data analysis.
In the last decade, inspired by the success of compressive sensing, there has been a
growing interest in incorporating sparsity to visual learning, such as image/video processing~\cite{Candes-2009-RPCA},
object classification~\cite{Wright-2009-face,Cabral-2011} and motion segmentation~\cite{Rao-2010-motion}.
Early studies~\cite{Elhamifar-2009-SSC,Wright-2009-face} usually consider the 1D sparsity (i.e., the nonzero entries of a vector, also known as the $l_0$ norm) in their models. Recently,
there has been a surge of methods~\cite{Candes-2009-RPCA,Liu-2010-LRR,Favaro-2011-CloseForm} which also consider the rank of a matrix as a 2D sparsity measure. However, it is difficult to directly solve these models due to the discrete nature of the $l_0$ norm and the rank function.
 A common strategy to alleviate this problem has been to use the $l_1$ norm and the nuclear norm~\cite{Recht-2007-GuaranteeNNM} as the convex surrogates of the $l_0$ norm and the rank function, respectively.

An important problem in unsupervised learning of visual data is subspace clustering. Recent advances in subspace clustering
make use of sparsity-based techniques. For example, sparse subspace clustering (SSC)~\cite{Elhamifar-2009-SSC,Elhamifar-2010-SSC,Soltanolkotabi-2011-SSC} uses the 1D sparsest representation vectors produced by $l_1$ norm minimization to define the affinity matrix of an undirected graph. Then subspace clustering is performed by spectral clustering techniques, such as normalized cut (NCut)~\cite{Shi-2000-ncut}. However, as SSC computes the sparsest representation of each points individually, there is no global structural constraint on the affinity matrix. This characteristic can degrade the clustering performance when data is grossly corrupted. Moreover, according to the theoretical work of~\cite{Nasihatkon-2011-SSC}, the within subspace connectivity assumption for SSC holds only for 2- and 3-dimensional subspaces. So SSC may probably over-segment subspaces when the dimensions are higher than 3.

Low-rank representation (LRR)~\cite{Liu-2010-LRR,Favaro-2011-CloseForm,Ni-2010-LRRPSD} is another recently proposed sparsity-based subspace clustering model. The intuition behind LRR is to learn a low-rank representation of the data.
The work by~\cite{Liu-2011-LRR} shows that LRR is intrinsically equivalent to the shape interaction matrix (SIM)~\cite{Jo-1998-Factor} in absence of noise. In this case, LRR can reveal the true clustering when the subspaces are independent and the data sampling is sufficient\footnote{The subspaces are independent if and only if the dimension of their direct sum is equal to the sum of their dimensions~\cite{Liu-2011-LRR}. For each subspace, the data sampling is sufficient if and only if the rank of the data matrix is equal to the dimension of the subspace~\cite{Liu-2011-Latent}.}. However, LRR suffers from some limitations as well. First, the nuclear norm minimization in LRR typically requires to calculate the singular value decomposition (SVD) at each iteration, which  becomes computationally impractical as the
 scale of the problem grows. By combining a linearized version of alternating direction method (ADM)~\cite{Lin-2009-IALM} with an acceleration technique for SVD computation, the work in~\cite{Lin-2011-LADM} proposed a fast solver, which significantly improves the speed for solving LRR. However, the SVD computation still cannot be completely avoided. Second, and more importantly, if the observations are insufficient, LRR (also SSC) may result in a degenerate solution that significantly degrades the clustering performance.
The work in~\cite{Liu-2011-Latent} introduces ``hidden effects'' to overcome this drawback. However, it is  unclear whether
such ``hidden effects'' can recover the multiple subspace structure for clustering. Moreover, introducing latent variables makes the problem more complex and hard to optimize.


The insufficient data sampling problem in SSC and LRR is similar in spirit to the small sample size problem, that is common in some subspace learning methods, such as linear discriminant analysis~\cite{fisher} and canonical correlation analysis~\cite{hotelling_cca}. In these methods, if the number of samples is smaller than the dimension of the features, the covariance matrices are rank deficient. Three are the common approaches
to solve this problem~\cite{Torre2012}: dimensionality reduction, regularization and factorization (i.e., explicitly parameterize the projection matrix as the product of low-rank matrices). In this paper, we incorporate the factorization idea into representation learning and propose
fixed-rank representation (FRR) to partially solve the problems in existing unsupervised visual learning models. FRR has three main benefits:
\begin{itemize}
\item
Unlike SSC and LRR, which use the sparsest and lowest rank representations, FRR \emph{explicitly} parameterizes the representation matrix as the product of two low-rank matrices. When there is no noise and the data sampling is sufficient, we prove that the FRR solution is also the optimal solution to LRR. In this case, FRR can reveal the multiple subspace structure. Furthermore, we prove that under some suitable conditions, even when the data sampling is insufficient, the memberships of samples to each subspace still can be
identified by FRR. A sparse regularizer is introduced to FRR to model both small noises and gross outliers, which provides robustness
to FRR in real applications.

\item
The most expensive computational component in LRR is to perform SVD at each iteration. Even with some acceleration techniques, the scalability of the nuclear norm minimization is still limited by the computational complexity of SVD. In contrast, FRR avoids SVD computation and can be efficiently applied to large-scale problems.

\item
FRR can also be extended for unsupervised feature extraction. By considering a transposed version of FRR (TFRR), we
show that FRR is related to existing feature extraction methods, such as principal component analysis (PCA)~\cite{pearson,hotelling}. Indeed, our analysis provides a unified framework to understand single subspace feature extraction and multiple subspace clustering by analyzing the column and row spaces of the data.
\end{itemize}

\section{A Review of Previous Work}
Given a data set\footnote{Bold capital letters (e.g., $\mathbf{M}$) denote matrices. The range and the null spaces of $\mathbf{M}$ are defined as
$\mathcal{R}(\mathbf{M}):=\{\mathbf{a} | \exists \mathbf{b}, \mathbf{a}=\mathbf{M}\mathbf{b}\}$ and $\mathcal{N}(\mathbf{M}):=\{\mathbf{a} | \mathbf{M}\mathbf{a} = \mathbf{0}\}$, respectively. $[\mathbf{M}]_{ij}$ and $[\mathbf{M}]_{i}$ denote the $(i,j)$-th entry and
the $i$-th column of $\mathbf{M}$, respectively.
$\mathbf{M}^{\dag}$ denotes the Moore-Penrose pseudoinverse of $\mathbf{M}$.
The block-diagonal matrix formed by a collection of matrices $\mathbf{M}_1,\mathbf{M}_2,...,\mathbf{M}_k$
is denoted by $\mbox{diag}(\mathbf{M}_1,\mathbf{M}_2,...,\mathbf{M}_k)$. $\mathbf{1}_n$ is the all-one column vector of length $n$. $\mathbf{I}_n$ is the $n\times n$ identity matrix.
$\langle\cdot,\cdot\rangle$ denotes the inner product of two matrices. A variety of norms on matrix and vector will be used.
$\|\cdot\|_F$ is the Frobenius norm, $\|\cdot\|_*$ is the nuclear norm \cite{Recht-2007-GuaranteeNNM}, $\|\cdot\|_{2,1}$ is the $l_{2,1}$ norm \cite{Liu-2009-L21}, $\|\cdot\|$ is the spectral norm, $\|\cdot\|_1$, $\|\cdot\|_2$ and $\|\cdot\|_{\infty}$ are the $l_1$,
$l_2$ and $l_{\infty}$ norms, respectively.} $\mathbf{X}=[\mathbf{X}_1,\mathbf{X}_2,\cdots,\mathbf{X}_k] \in \mathbb{R}^{d\times n}$ drawn from a union of $k$ subspaces $\{\mathcal{C}_i\}_{i=1}^k$, where $\mathbf{X}_i$ is a collection of $n_i$ data points sampled from the subspace $\mathcal{C}_i$ with an unknown dimension $d_{C_i}$, the goal of subspace clustering is to cluster data points into their respective subspaces. This section provides a review of SSC and LRR for solving this problem.
To clearly understand the mechanism of these methods, we first consider the case when the data is noise-free. From now on, we always write $\mathbf{X}=\mathbf{U}_{X}\Sigma_{X}\mathbf{V}_{X}^T$ and $r_X$ as the compact SVD and the rank of $\mathbf{X}$, respectively.

\subsection{Sparse Subspace Clustering (SSC)}
SSC~\cite{Elhamifar-2009-SSC,Elhamifar-2010-SSC,Soltanolkotabi-2011-SSC} is based on the idea that each data point in the subspace $\mathcal{C}_i$ should be represented as a linear combination of other points that are also in $\mathcal{C}_i$. Using this intuition, SSC finds the sparsest representation coefficients $\mathbf{Z} = [[\mathbf{Z}]_1,[\mathbf{Z}]_2,\cdots,[\mathbf{Z}]_n]$ by considering the sequence of optimization problems
\begin{equation}
\min\limits_{[\mathbf{Z}]_i}\|[\mathbf{Z}]_i\|_1, \ s.t. \ [\mathbf{X}]_i = \mathbf{X}[\mathbf{Z}]_i, \ [\mathbf{Z}]_{ii} = 0, \label{eq:sr}
\end{equation}
where $i=1,2,\cdots,n.$
Then one can use $\mathbf{Z}$ to define the affinity matrix of an undirected graph as
$(|\mathbf{Z}| + |\mathbf{Z}^T|)$ and perform NCut on this graph, where $|\mathbf{Z}|$ denotes a matrix whose entries are the absolute values of $\mathbf{Z}$.
The SSC model can also be rewritten in matrix form as
\begin{equation}
\min\limits_{\mathbf{Z}}\|\mathbf{Z}\|_1, \ s.t. \ \mathbf{X} = \mathbf{X}\mathbf{Z}, \ [\mathbf{Z}]_{ii} = 0.\label{eq:gsr}
\end{equation}
Note that both $l_1$ norm minimization models (\ref{eq:sr}) and (\ref{eq:gsr}) can only be solved numerically.

\subsection{Low-Rank Representation (LRR)}\label{sec:lrr}

By extending the sparsity measure from 1D to 2D for the representation, LRR \cite{Liu-2010-LRR,Favaro-2011-CloseForm,Ni-2010-LRRPSD}
proposes a low-rank based criterion for subspace clustering.
By utilizing the nuclear norm as a surrogate for the rank function, LRR solves the following nuclear norm minimization problem
\begin{equation}
\min\limits_{\mathbf{Z}} \|\mathbf{Z}\|_*, \ s.t. \ \mathbf{X} = \mathbf{X}\mathbf{Z}.\label{eq:lrr}
\end{equation}
Unlike SSC, which can only be solved numerically, $\mathbf{V}_{X}\mathbf{V}_{X}^T$ (also known as SIM \cite{Jo-1998-Factor}), which has a block-diagonal structure, is the closed-form solution to (\ref{eq:lrr}) \cite{Liu-2011-LRR}. Although \cite{Liu-2011-LRR} has proved this, in the following section, we will provide a simpler
derivation, that provides new insights into LRR.

\section{Fixed-Rank Representation}\label{sec:frr}
In this section, we propose a new model, named fixed-rank representation (FRR), for subspace clustering. We start with the following analysis on LRR.

\subsection{Motivation}

To better understand the mechanism of LRR and illustrate our motivation, we show that $\mathbf{V}_{X}\mathbf{V}_{X}^T\in\mathcal{R}(\mathbf{X}^T)$ is the optimal solution to LRR in a simple way\footnote{Note that here we only analyze the optimality of $\mathbf{V}_{X}\mathbf{V}_{X}^T$ to (\ref{eq:lrr}), not its uniqueness.}.
By the identity $\mathbf{X}=\mathbf{X}\mathbf{X}^{\dag}\mathbf{X}$ and the constraint in (\ref{eq:lrr}), we have
$\mathbf{X}=\mathbf{X}\mathbf{Z}=\mathbf{X}\mathbf{X}^{\dag}\mathbf{X}\mathbf{Z}=\mathbf{X}\mathbf{X}^{\dag}\mathbf{X}$. Thus $\mathbf{X}^{\dag}\mathbf{X}=\mathbf{V}_X\mathbf{V}_X^T$ is a feasible solution to (\ref{eq:lrr}). So the general form of the solution is $\mathbf{Z}=\mathbf{V}_X\mathbf{V}_X^T + \mathbf{Z}_n$, where $\mathbf{Z}_n \in \mathcal{N}(\mathbf{X})$. As
$\mathcal{R}(\mathbf{X}^T)\perp\mathcal{N}(\mathbf{X})$, we have $\mathbf{V}_X^T\mathbf{Z}_n=\mathbf{0}$. This together with the duality definition of nuclear norm \cite{Recht-2007-GuaranteeNNM} leads the following inequality
$$
\|\mathbf{Z}\|_*=\max\limits_{\|\mathbf{Y}\|\leq 1}\langle\mathbf{Z},\mathbf{Y}\rangle\geq \langle\mathbf{Z},\mathbf{V}_X\mathbf{V}_X^T\rangle
= r_X =\|\mathbf{V}_X\mathbf{V}_X^T\|_*.
$$
This concludes that $\mathbf{V}_X\mathbf{V}_X^T$ is the minimizer to (\ref{eq:lrr}).

The first observation from the prevous analysis is that LRR can successfully remove the effects from $\mathcal{N}(\mathbf{X})$ to obtain a block-diagonal matrix when the data sampling is sufficient. However, it is also observed that the ``lowest rank'' representation in LRR is actually the largest rank matrix within the row space of $\mathbf{X}$, namely the rank of this representation is always equal to the dimension of the row space. Therefore, the lack of observations for each subspace may significantly degrade the clustering performance. For example, due to insufficient data sampling, the dimension of the row space may be equal to the number of samples (i.e., $r_X = n \leq d$). In this case, the optimal solution to~(\ref{eq:lrr}) may reduce to an identity matrix and thus LRR may fail. See Fig.~\ref{fig:structure} as an example.

An obvious question is whether we can find a lower rank representation in the row space of the data set to exactly reveal the subspace memberships for clustering, even when the data sampling is insufficient. In the following subsection, we give a positive answer to this question.

\subsection{The Basic Model}

The key idea of FRR is to minimize the Frobenius norm of the representation $\mathbf{Z}$ instead of the nuclear norm as in LRR. FRR simultaneously computes a fixed lower rank representation $\tilde{\mathbf{Z}}$ (hereafter we write $\mbox{rank}(\tilde{\mathbf{Z}})=m$).
That is, we jointly optimize $\mathbf{Z}$ and $\tilde{\mathbf{Z}}$ as
\begin{equation}
\min\limits_{\mathbf{Z}, \tilde{\mathbf{Z}}}\|\mathbf{Z}-\tilde{\mathbf{Z}}\|_F^2, \ s.t. \ \mathbf{X}=\mathbf{X}\mathbf{Z}, \ \mbox{rank}(\tilde{\mathbf{Z}})=m.
\end{equation}
Obviously, $\tilde{\mathbf{Z}}$ can be expressed, non-uniquely, as a matrix product $\tilde{\mathbf{Z}}=\mathbf{L}\mathbf{R}$,
where $\mathbf{L} \in \mathbb{R}^{n\times m}$ and $\mathbf{R} \in \mathbb{R}^{m\times n}$. Replacing $\tilde{\mathbf{Z}}$ by $\mathbf{L}\mathbf{R}$,
we arrive at our basic FRR model
\begin{equation}
\min\limits_{\mathbf{Z},\mathbf{L,\mathbf{R}}}\|\mathbf{Z}-\mathbf{L}\mathbf{R}\|_F^2, \ s.t. \ \mathbf{X} = \mathbf{X}\mathbf{Z}.\label{eq:frr}
\end{equation}
In the following sections, we will analyze the problem (\ref{eq:frr}), show properties of the solution to (\ref{eq:frr}), and extend it for real applications.

\subsection{Analysis on the Basic Model}

At first sight, the factorization of $\tilde{\mathbf{Z}}$ leads to a non-convex optimization problem which may prevent one from getting a global solution. The difficulty results from the fact that the minimizer is non-unique. Fortunately, in the following theorem, we prove that one can always obtain a globally optimal solution to (\ref{eq:frr})
in closed-form.
\begin{theorem}\label{thm:frr}
Let $[\mathbf{V}_{X}]_{1:m}=[[\mathbf{V}_{X}]_{1},[\mathbf{V}_{X}]_{2},\cdots,[\mathbf{V}_{X}]_{m}]$.
Then for any fixed $m \leq r_X$,
$(\mathbf{Z}^*,\mathbf{L}^*,\mathbf{R}^*):=(\mathbf{V}_{X}\mathbf{V}_{X}^T, [\mathbf{V}_{X}]_{1:m}, [\mathbf{V}_{X}]_{1:m}^T)$ is a globally optimal solution
to~(\ref{eq:frr}) and the minimum objective function value is $(r_X-m)$.
\end{theorem}
The proof of this theorem is based on the following lemma.
\begin{lemma}(Courant-Fischer Minimax Theorem \cite{Golub-1996-MatrixCom})\label{lem:cfmt}
For any symmetric matrix $\mathbf{A}\in \mathbb{R}^{n\times n}$, we have that
$$
\lambda_i(\mathbf{A})=\max\limits_{\dim(\mathcal{S})=i}
\min\limits_{\mathbf{0}\neq\mathbf{y}\in\mathcal{S}}\mathbf{y}^T\mathbf{A}\mathbf{y}/\mathbf{y}^T\mathbf{y}, \ \mbox{for} \ i = 1,2,...,n,
$$
where $\mathcal{S}\subset \mathbb{R}^n$ is some subspace
and $\lambda_i(\mathbf{A})$ is the $i$-th largest eigenvalue of $\mathbf{A}$.
\end{lemma}

\begin{proof}
First, by the well known Eckart-Young theorem \cite{Eckart-1936-EY}, given $\mathbf{Z}$, we have
\begin{equation}
\min\limits_{\mathbf{L},\mathbf{R}}\|\mathbf{Z}-\mathbf{L}\mathbf{R}\|_F^2
=\sum\limits_{i = m+1}^d\sigma_i^2(\mathbf{Z}),\label{eq:eyt}
\end{equation}
where $\sigma_i(\mathbf{Z})$ is the $i$-th largest singular value of $\mathbf{Z}$. Now we prove that
\begin{equation}
\mbox{if} \ \mathbf{X}=\mathbf{X}\mathbf{Z} \ \mbox{then} \ \sigma_{r_X}(\mathbf{Z}) \geq 1.\label{eq:sigma}
\end{equation}
By $\mathbf{X}=\mathbf{X}\mathbf{Z}$, we have that $\mbox{rank}(\mathbf{Z}) \geq r_X$.
Then (\ref{eq:eyt}) and (\ref{eq:sigma}) imply that the minimum objective function value is no less than $r_X-m$.
Indeed, by the compact SVD of $\mathbf{X}$ and $\mathbf{X}=\mathbf{X}\mathbf{Z}$, we have
\begin{equation}
\mathbf{V}_{X}^T = \mathbf{V}_{X}^T\mathbf{Z},\label{eq:uzu}
\end{equation}
By Lemma~\ref{lem:cfmt},
$\sigma_i(\mathbf{Z})=\max\limits_{\dim(\mathcal{S})=i}\min\limits_{\mathbf{0}\neq\mathbf{y}\in\mathcal{S}}
\|\mathbf{Z}^T\mathbf{y}\|_2/\|\mathbf{y}\|_2$, where $\|\cdot\|_2$ is the $l_2$ norm of a vector.
So by choosing $\mathcal{S}=\mathcal{R}(\mathbf{V}_{X})$ and utilizing (\ref{eq:uzu}),
\begin{equation}
\begin{array}{rcl}
\sigma_{r_X}(\mathbf{Z}) & \geq & \min\limits_{\mathbf{0}\neq\mathbf{y}\in
\mathcal{R}(\mathbf{V}_{X})}\|\mathbf{Z}^T\mathbf{y}\|_2/\|\mathbf{y}\|_2\\
& = & \min\limits_{\mathbf{b}\neq \mathbf{0}}\|\mathbf{Z}^T\mathbf{V}_{X}\mathbf{b}\|_2/\|\mathbf{V}_{X}\mathbf{b}\|_2\\
& = & \min\limits_{\mathbf{b}\neq \mathbf{0}}\|\mathbf{V}_{X}\mathbf{b}\|_2/\|\mathbf{V}_{X}\mathbf{b}\|_2=1.\\
\end{array}
\end{equation}
Next, when $\mathbf{Z}=\mathbf{V}_{X}\mathbf{V}_{X}^T$, it can be easily checked that
the objective function value is $(r_X-m)$. Again, by Eckart-Young theorem, $\mathbf{L}\mathbf{R}=[\mathbf{V}_{X}]_{1:m}[\mathbf{V}_{X}]_{1:m}^T$.
Thus we have $(\mathbf{V}_{X}\mathbf{V}_{X}^T, [\mathbf{V}_{X}]_{1:m}, [\mathbf{V}_{X}]_{1:m}^T)$ is a globally
optimal solution to (\ref{eq:frr}), thereby completing the proof of the theorem.
\end{proof}

Based on Theorem~\ref{thm:frr}, we can derive the following corollary to illustrate the structure of the optimal solution to (\ref{eq:frr}).
\begin{corollary}\label{cor:frr1}
Under the assumption that subspaces are independent and data $\mathbf{X}$ is clean, there exists a globally optimal
solution ($\mathbf{Z}^*,\mathbf{L}^*,\mathbf{R}^*$) to problem (\ref{eq:frr}) with the following
structure:
\begin{equation}
\mathbf{Z}^* = \mbox{diag}(\mathbf{Z}_1,\mathbf{Z}_2,...,\mathbf{Z}_k),\label{eq:block}
\end{equation}
where $\mathbf{Z}_i$ is an $n_i\times n_i$ matrix with $\mbox{rank}(\mathbf{Z}_i)=d_{C_i}$ and
\begin{equation}
\mathbf{L}^*{\mathbf{R}^*}\in\mathcal{R}(\mathbf{Z}^*)=\mathcal{R}(\mathbf{X}^T).\label{eq:lrz}
\end{equation}
\end{corollary}
The proof of this corollary is based on the following lemma.

\begin{lemma}\cite{Jo-1998-Factor}\label{lem:sim}
Let $\mathbf{X}=\mathbf{U}_{X}\Sigma_{X}\mathbf{V}_{X}^T$  be the compact SVD.
Under the same assumption in Corollary~\ref{cor:frr1}, $\mathbf{V}_{X}\mathbf{V}_{X}^T$ is a block diagonal matrix
that has exactly $k$ blocks. Moreover, the $i$-th block on its diagonal is an $n_i\times n_i$ matrix with rank $d_{C_i}$.
\end{lemma}

\begin{proof}
By the proof of Theorem~\ref{thm:frr}, we have that $\mathbf{Z}^*=\mathbf{V}_{X}\mathbf{V}_{X}^T$
is a global optimal solution to (\ref{eq:frr}) and any global optimal $\mathbf{L}^*$ and $\mathbf{R}^*$
are in the range space $\mathcal{R}(\mathbf{Z}^*)$. So we have that
$\mathbf{L}^*\mathbf{R}^*\in\mathcal{R}(\mathbf{Z}^*)=\mathcal{R}(\mathbf{X}^T)$.
By Lemma~\ref{lem:sim}, we achieve the block diagonal structure (\ref{eq:block}) for $\mathbf{Z}^*$, which concludes the proof.
\end{proof}

However, such $\mathbf{Z}^*$ suffers from the same limitation of LRR. Namely,
when the data sampling is insufficient, $\mathbf{Z}^*$ will probably degenerate and thus the clustering may fail.

Fortunately, as shown in (\ref{eq:lrz}), $\mathbf{L}^*\mathbf{R}^*$ can still be spanned by the row space of $\mathbf{X}$.
This inspires us to consider this lower rank representation for subspace clustering.
\begin{corollary}\label{cor:frr2}
Assuming that the columns of $\mathbf{Z}^*$ are normalized (i.e. $\mathbf{1}_n^T\mathbf{Z}^* = \mathbf{1}_n^T$) and fix $m = k$,
then there exists globally optimal $\mathbf{L}^*$ and $\mathbf{R}^*$ to problem (\ref{eq:frr}) such that
\begin{equation}
\mathbf{L}^*\mathbf{R}^*=\mbox{diag}(n_1\mathbf{1}_{n_1}\mathbf{1}_{n_1}^T,n_2\mathbf{1}_{n_2}\mathbf{1}_{n_2}^T,...,
n_k\mathbf{1}_{n_k}\mathbf{1}_{n_k}^T).\label{eq:lrblock}
\end{equation}
\end{corollary}
\textbf{Remark:} Corollary~\ref{cor:frr2} does not guarantee that an arbitrary rank-$k$ optimal solution has the block-diagonal structure
(\ref{eq:lrblock}) due to the non-unique of the minimizer $(\mathbf{L}^*, \mathbf{R}^*)$. However, in our experiments,
we have observed that empirically choosing the first $k$ columns of $\mathbf{V}_X$ works well on the tested data (e.g., Fig. \ref{fig:structure}).

\begin{proof}
By Corollary~\ref{cor:frr1} and the normalization assumption,
$\mathbf{Z}^* = diag(\mathbf{Z}_1^*,\mathbf{Z}_2^*,...,\mathbf{Z}_k^*)$,
where $\mathbf{Z}_i^*$ is an $n_i\times n_i$ for subspace $\mathcal{C}_i$ and
$\mathbf{1}_{n_i}$ is an eigenvector of $\mathbf{Z}_i^*$ with eigenvalue $1$. Thus there exists a basis $\mathbf{H}=[\mathbf{h}_1,\mathbf{h}_2,...,\mathbf{h}_k]$,
each vector of which with the form $\mathbf{h}_i = [\mathbf{0},\mathbf{1}_{n_i}^T,\mathbf{0}]^T$
is eigenvector of $\mathbf{Z}$ with eigenvalue $1$.
By the Eckart-Young theorem (similar to the proof of Theorem~\ref{thm:frr}),
we have that $\mathbf{L}^*=\mathbf{H}$ and $\mathbf{R}^*=\mathbf{H}^T$ are global optimal solutions to (\ref{eq:frr}),
which directly leads (\ref{eq:lrblock}).
\end{proof}

In principle, the normalization of $\mathbf{Z}^*$ could be considered as a strong assumption, hence
it cannot always be guaranteed in real situations. Therefore, we explicitly enforce each column of $\mathbf{Z}$ to sum to one
\begin{equation}
\min\limits_{\mathbf{Z},\mathbf{L,\mathbf{R}}}\|\mathbf{Z}-\mathbf{L}\mathbf{R}\|_F^2, \ s.t. \ \mathbf{X} = \mathbf{X}\mathbf{Z},
\ \mathbf{1}_n^T\mathbf{Z} = \mathbf{1}_n^T.\label{eq:frrn}
\end{equation}

\subsection{Sparse Regularization for Corruptions}\label{sec:robust}
In real applications, the data are often corrupted by both small noises and gross outliers.
In the following, we show how to extend problem (\ref{eq:frrn}) to deal with corruptions.
By modeling corruptions as a new term $\mathbf{E}$, we consider the following regularized optimization problem
\begin{equation}
\begin{array}{c}
\min\limits_{\mathbf{Z}, \mathbf{L}, \mathbf{R}, \mathbf{E}}\|\mathbf{Z}-\mathbf{L}\mathbf{R}\|_F^2 + \mu\|\mathbf{E}\|_{s},\\
\ s.t. \ \mathbf{X}=\mathbf{X}\mathbf{Z} + \mathbf{E}, \ \mathbf{1}_n^T\mathbf{Z}=\mathbf{1}_n^T,\label{eq:rfrrn}
\end{array}
\end{equation}
where the parameter $\mu >0$ is used to balance the effects of the two terms and $\|\cdot\|_s$ is a sparse norm
corresponding to our assumption on $\mathbf{E}$.
Here we adopt the $l_{2,1}$ norm
to characterize the corruptions since it can successfully identify the indices of the outliers and remove small noises \cite{Liu-2011-LRR1}.
Algorithm~\ref{alg:frr} summarizes the whole FRR based subspace clustering framework.
\begin{algorithm}
   \caption{FRR for Subspace Clustering}
\begin{algorithmic}
   \STATE {\bfseries Input:} Let $\mathbf{X} \in \mathbb{R}^{d \times n}$ be a set of
   data points sampled from $k$ subspaces.
   \STATE {\bfseries Step 1:} Solve (\ref{eq:rfrrn}) to obtain ($\mathbf{Z}^*$, $\mathbf{L}^*$, $\mathbf{R}^*$).
   \STATE {\bfseries Step 2:} Construct a graph by using $(|\mathbf{Z}^*|+|(\mathbf{Z}^*)^T|)$ or
   $(|\mathbf{L}^*\mathbf{R}^*|+|(\mathbf{L}^*\mathbf{R}^*)^T|)$ as the affinity matrix.
   \STATE {\bfseries Step 3:} Apply NCut to this graph to obtain the clustering.
\end{algorithmic}\label{alg:frr}
\end{algorithm}

\section{Extending FRR for Feature Extraction}\label{sec:tfrr}
Besides subspace clustering, the mechanism of FRR can also be applied for feature extraction.
That is, one can recover the column space of the data set by solving the following transposed FRR (TFRR)
\begin{equation}
\min\limits_{\mathbf{Z},\mathbf{L,\mathbf{R}}}\|\mathbf{Z}-\mathbf{L}\mathbf{R}\|_F^2, \ s.t. \ \mathbf{X} = \mathbf{Z}\mathbf{X},\label{eq:tfrr}
\end{equation}
where $m\leq r_X$, $\mathbf{L} \in \mathbb{R}^{d\times m}$, $\mathbf{R} \in \mathbb{R}^{m \times d}$ and $\mathbf{Z}\in\mathbb{R}^{d\times d}$.
For noisy data, by using similar techniques as in Section~\ref{sec:robust}, we introduce an explicit corruption term $\mathbf{E}$
into the objective function and the constraint. Hence we obtain the
robust version of TFRR for feature extraction
\begin{equation}
\min\limits_{\mathbf{Z}, \mathbf{L}, \mathbf{R}, \mathbf{E}}\|\mathbf{Z}-\mathbf{L}\mathbf{R}\|_F^2 + \mu\|\mathbf{E}\|_{s},
\ s.t. \ \mathbf{X}=\mathbf{Z}\mathbf{X} + \mathbf{E}.\label{eq:trfrr}
\end{equation}

\subsection{Relationship to Principal Component Analysis}
Principal component analysis (PCA) is one of the most popular dimensionality reduction techniques~\cite{pearson,hotelling}. The basic ideas
behind PCA date back to Pearson in 1901~\cite{pearson}, and a more
general procedure was described by Hotelling~\cite{hotelling} in
1933. There are several energy functions which lead to subspace spanned by the principal components~\cite{Torre2012}. For instance, PCA finds the matrix  $\mathbf{P}\in\mathbb{R}^{d\times m}$ that minimizes:
\begin{equation}
\min\limits_{\mathbf{P}}\|\mathbf{X}-\mathbf{P}\mathbf{P}^T\mathbf{X}\|_F^2, \ s.t. \ \mathbf{P}^T\mathbf{P} = \mathbf{I}_m.\label{eq:pca}
\end{equation}
It can be shown that $\mathbf{P}^*=[\mathbf{U}_{X}]_{1:m}$ is the optimal solution to (\ref{eq:pca}), where
$[\mathbf{U}_{X}]_{1:m}=[[\mathbf{U}_{X}]_{1},[\mathbf{U}_{X}]_{2},\cdots,[\mathbf{U}_{X}]_{m}]$.
The following corollary shows that the mechanism of TFRR can also be applied to formulate PCA.
\begin{corollary}
For any fixed $m \leq r_X$,
$(\mathbf{Z}^*,\mathbf{L}^*,\mathbf{R}^*):=(\mathbf{U}_{X}\mathbf{U}_{X}^T, [\mathbf{U}_{X}]_{1:m}, [\mathbf{U}_{X}]_{1:m}^T)$ is a globally optimal solution
to~(\ref{eq:tfrr}) and the minimum objective function value is $(r_X-m)$.
\end{corollary}
\begin{proof}
The proof of Theorem~\ref{thm:frr} directly leads to the above corollary.
\end{proof}

\section{Optimization for FRR}\label{sec:solver}
In this section, we develop a fast numerical solver for FRR related models by extending the
classic alternating direction method (ADM) \cite{Lin-2009-IALM} to non-convex problems. To solve the problem~(\ref{eq:rfrrn})\footnote{As other
FRR related models can be solved in similar way, we do not further explore them in this section.},
we introduce Lagrange multipliers $\Lambda$ and $\Pi$ to remove the equality constraints. The resulting augmented Lagrangian function is
\begin{equation}
\begin{array}{ll}
\mathcal{L}_A(\mathbf{Z}, \mathbf{L}, \mathbf{R}, \mathbf{E}, \Lambda, \Pi)
=\|\mathbf{Z}-\mathbf{L}\mathbf{R}\|_F^2 + \mu\|\mathbf{E}\|_{2,1} \\
+ \langle \Lambda, \mathbf{X}-\mathbf{X}\mathbf{Z}-\mathbf{E}\rangle + \langle \Pi, \mathbf{1}_n^T\mathbf{Z}-\mathbf{1}^T\rangle \\
+ \frac{\beta}{2}(\|\mathbf{X}-\mathbf{X}\mathbf{Z}-\mathbf{E}\|_F^2+\|\mathbf{1}_n^T\mathbf{Z}-\mathbf{1}_n\|_F^2),
\end{array}\label{eq:alf}
\end{equation}
where $\beta > 0$ is a penalty parameter.
It is important to note that although (\ref{eq:alf}) is not jointly convex for all variables, it is convex with respect
to each variable while fixing the others. This property allows the iteration scheme to be well defined.
So we minimize (\ref{eq:alf}) with respect to $\mathbf{L}$, $\mathbf{R}$, $\mathbf{Z}$,
and $\mathbf{E}$ one at a time while fixing the others at their latest values, and then update the Lagrange multipliers $\Lambda$ and $\Pi$:
\begin{eqnarray}
\mathbf{L}_{+}& \leftarrow  &\mathbf{Z}\mathbf{R}^{\dag}\equiv\mathbf{Z}\mathbf{R}^T(\mathbf{R}\mathbf{R}^T)^{\dag},\label{eq:l}\\
\mathbf{R}_{+}& \leftarrow  &\mathbf{L}_+^{\dag}\mathbf{Z}\equiv(\mathbf{L}_+^T\mathbf{L}_+)^{\dag}\mathbf{L}_+^{T}\mathbf{Z},\label{eq:r}\\
\mathbf{Z}_{+}& \leftarrow &(2\mathbf{I}_n+\beta(\mathbf{X}^T\mathbf{X}+\mathbf{1}_n\mathbf{1}_n^T))^{-1}\mathbf{B},\label{eq:z}\\
\mathbf{E}_{+}& \leftarrow & \mathop{\arg\min}\limits_{\mathbf{E}}\mu\|\mathbf{E}\|_{2,1}+\frac{\beta}{2}\|\mathbf{C}-\mathbf{E}\|_F^2,\label{eq:e}\\
\Lambda_{+}& \leftarrow & \Lambda + \beta(\mathbf{X}-\mathbf{X}\mathbf{Z}_{+}-\mathbf{E}_{+}),\\
\Pi_{+}& \leftarrow & \Pi + \beta(\mathbf{1}_n^T\mathbf{Z}_+-\mathbf{1}_{n}^T),\\
\beta_{+}& \leftarrow & \min(\bar{\beta},\rho\beta),\label{eq:beta}
\end{eqnarray}
where the subscript $+$ denotes that the values are updated, $\bar{\beta}$ is the upper bound of $\beta$, $\rho > 1$ is the step length parameter, $\mathbf{B}=
2\mathbf{L}_{+}\mathbf{R}_{+} + \beta(\mathbf{X}^T\mathbf{X} - \mathbf{X}^T(\mathbf{E}-\Lambda/\beta)) + \beta\mathbf{1}_n\mathbf{1}_n^T-\mathbf{1}_n\Pi$
and $\mathbf{C}=\mathbf{X}-\mathbf{X}\mathbf{Z}_{+} + \Lambda/\beta$.
The subproblem (\ref{eq:e}) can be solved by Lemma 3.2 in \cite{Liu-2010-LRR}.
We then reduce the computational cost for solving (\ref{eq:l}) and (\ref{eq:r}). It follows from~(\ref{eq:r}) that
\begin{equation}
\mathbf{L}_+\mathbf{R}_+ = \mathbf{L}_+(\mathbf{L}_+^T\mathbf{L}_+)^{\dag}\mathbf{L}_+^{T}\mathbf{Z} = \mathcal{P}_{L_+}(\mathbf{Z}).
\end{equation}
By considering the compact SVD: $\mathbf{R}=\mathbf{U}_{R_r}\Sigma_{R_r}\mathbf{V}_{R_r}^T$, we have
$\mathbf{L}_+ = \mathbf{Z}\mathbf{V}_{R_r}\Sigma_{R_r}^{-1}\mathbf{U}_{R_r}^T$ and $\mathbf{Z}\mathbf{R}^T=\mathbf{Z}\mathbf{V}_{R_r}\Sigma_{R_r}\mathbf{U}_{R_r}^T$. This implies
that $\mathcal{R}(\mathbf{L}_+) = \mathcal{R}(\mathbf{Z}\mathbf{R}^T)=\mathcal{R}(\mathbf{Z}\mathbf{V}_{R_r})$ and
\begin{equation}
\mathbf{L}_+\mathbf{R}_+ = \mathcal{P}_{ZR^T}(\mathbf{Z}),\label{eq:range}
\end{equation}
where $\mathcal{P}_{ZR^T}$ is the orthogonal projection into $\mathcal{R}(\mathbf{Z}\mathbf{R}^T)$.
Since the objective function of (\ref{eq:rfrrn}) depends on the product $\mathbf{L}_+\mathbf{R}_+$, different values of $\mathbf{L}_+$
and $\mathbf{R}_+$ are essentially equivalent as long as they give the same product. The identity~(\ref{eq:range}) shows that the inversion $(\mathbf{R}\mathbf{R}^T)^{\dag}$
and $(\mathbf{L}_+^T\mathbf{L}_+)^{\dag}$ can be saved when the projection $\mathcal{P}_{ZR^T}$ is computed. Specifically, one can compute
$\mathcal{P}_{ZR^T}=\mathbf{Q}\mathbf{Q}^T$,
where $\mathbf{Q}$ is the QR factorization of $\mathbf{Z}\mathbf{R}^T$. Then we have $\mathbf{L}_+\mathbf{R}_+ = \mathbf{Q}\mathbf{Q}^T\mathbf{Z}$ and one can derive:
\begin{eqnarray}
\mathbf{L}_+ &\leftarrow& \mathbf{Q},\label{eq:l1}\\
\mathbf{R}_+ &\leftarrow& \mathbf{Q}^T\mathbf{Z}.\label{eq:r1}
\end{eqnarray}
The schemes (\ref{eq:l1}) and (\ref{eq:r1}) are often preferred since computing (\ref{eq:r1}) by QR factorization is generally more stable than
solving the normal equations~\cite{Shen-2011-LRF}.
The complete algorithm is summarized in Algorithm~\ref{alg:adm}.
\begin{algorithm}
   \caption{Solving (\ref{eq:rfrrn}) by ADM-type Algorithm}
\begin{algorithmic}
   \STATE {\bfseries Input:} Observation matrix $\mathbf{X} \in \mathbb{R}^{d \times n}$, $m>0$, $\epsilon_1, \epsilon_2 > 0$, parameters $\beta > 0$ and $\rho > 1$.
   \STATE {\bfseries Initialization:} Initialize $\mathbf{Z}_0\in\mathbb{R}^{n\times n}$, $\mathbf{L}_0\in\mathbb{R}^{n\times m}$, $\mathbf{R}_0\in \mathbb{R}^{m\times n}$, $\mathbf{E}_0\in\mathbb{R}^{d\times n}$, $\Lambda_0 \in\mathbb{R}^{d\times n}$ and $\Pi_0 \in \mathbb{R}^{1\times n}$.
   \WHILE {not converged}
   \STATE {\bfseries Step 1:} Update ($\mathbf{Z}$, $\mathbf{L}$, $\mathbf{R}$, $\mathbf{E}$, $\Lambda$, $\Pi$) by (\ref{eq:l1}), (\ref{eq:r1}) and (\ref{eq:z})--(\ref{eq:beta}).
   \STATE {\bfseries Step 2:} Check the convergence conditions:\\
   $\|\mathbf{X}-\mathbf{X}\mathbf{Z}_+-\mathbf{E}_+\|_{\infty}\leq \epsilon_1$ and
   $\|\mathbf{1}_n^T\mathbf{Z}_+-\mathbf{1}_n^T\|_{\infty}\leq \epsilon_2$.
   \ENDWHILE
   \STATE {\bfseries Output:} $\mathbf{Z}^*$, $\mathbf{L}^*$, $\mathbf{R}^*$ and $\mathbf{E}^*$.
\end{algorithmic}\label{alg:adm}
\end{algorithm}

\section{Experimental Results}\label{sec:exp}
This section compared the performance of FRR against state-of-the-art algorithms on both subspace clustering and feature extraction.  All experiments are performed on a notebook computer with an Intel Core i7 CPU at 2.00 GHz and 6GB of memory, running Windows 7 and Matlab version 7.10.

\subsection{Subspace Clustering}
We first consider the subspace clustering problem, and compare the clustering performance and computational speed of FRR
to existing state-of-the-art methods, such as SIM, Random Sample Consensus (RANSAC)~\cite{Fischler-1981-RANSAC}, Local Subspace Analysis (LSA)~\cite{Yan-2006-LSA}, SSC and LRR.
As shown in Section~\ref{sec:frr}, both $\mathbf{Z}$ and $\mathbf{L}\mathbf{R}$ can
be utilized for clustering, we call these two strategies FRR$_1$ and FRR$_2$, respectively.

\subsubsection{Synthetic Data}
We performed subspace clustering on synthetic data to illustrate the insufficient data sampling problem (to verify the analysis in Section~\ref{sec:frr}). Let $k$, $p$, $d_h$ and $d_l$ denote the number of subspaces, the number of points in each subspace, the features (i.e., observed dimension) and the intrinsic dimension of the subspace, respectively. Then the data set, parameterized as ($k, p, d_h, d_l$), is generated by the same procedure in \cite{Liu-2010-LRR}:
$k$ independent subspaces $\{\mathcal{C}_i\}_{i=1}^{k}$ are constructed, whose basis
$\{\mathbf{U}\}_{i=1}^{k}$ are computed by $\mathbf{U}_{i+1}=\mathbf{T}\mathbf{U}_{i}$,
$1 \leq i \leq k-1$, where $\mathbf{T}$ is a random rotation and $\mathbf{U}_1$ is a random
column orthogonal matrix of dimension $d_h \times d_l$. Then we construct
a $d_h \times kp$ data matrix $\mathbf{X}=[\mathbf{X}_1,\mathbf{X}_2,...,\mathbf{X}_{k}]$ by sampling
$p$ data vectors from each subspace by $\mathbf{X}_i=\mathbf{U}_i\mathbf{C}_i$,
$1\leq i\leq k$, with $\mathbf{C}_i$ being a $d_l \times p$ matrix with uniform distribution. To generate the point set for
insufficient data sampling clustering, we fix
$k = 10$, $d_h = 100$ and $d_l = 50$ and vary $p \in [10,30]$. In this way, the number of samples in each subspace (at most 30) is less than the
intrinsic dimension (50 for each subspace).

Fig.~\ref{fig:structure} illustrated the structures of $\mathbf{Z}=\mathbf{V}_{X}\mathbf{V}_{X}^T$
and $\mathbf{L}\mathbf{R}= [\mathbf{V}_{X}]_{1:k}[\mathbf{V}_{X}]_{1:k}^T$ when $p = 10$. Since the data sampling
is insufficient, the optimal $\mathbf{Z}$ for (\ref{eq:lrr}) and (\ref{eq:frr}) reduces to $\mathbf{I}_n$ (see Fig.~\ref{fig:structure} (a)). In contrast, $\mathbf{L}\mathbf{R}$ can successfully reveal the multiple subspace structure (see Fig.~\ref{fig:structure} (b)).
\begin{figure}[t]
\center
\begin{tabular}{cc}
\includegraphics[width=0.4\textwidth,
keepaspectratio]{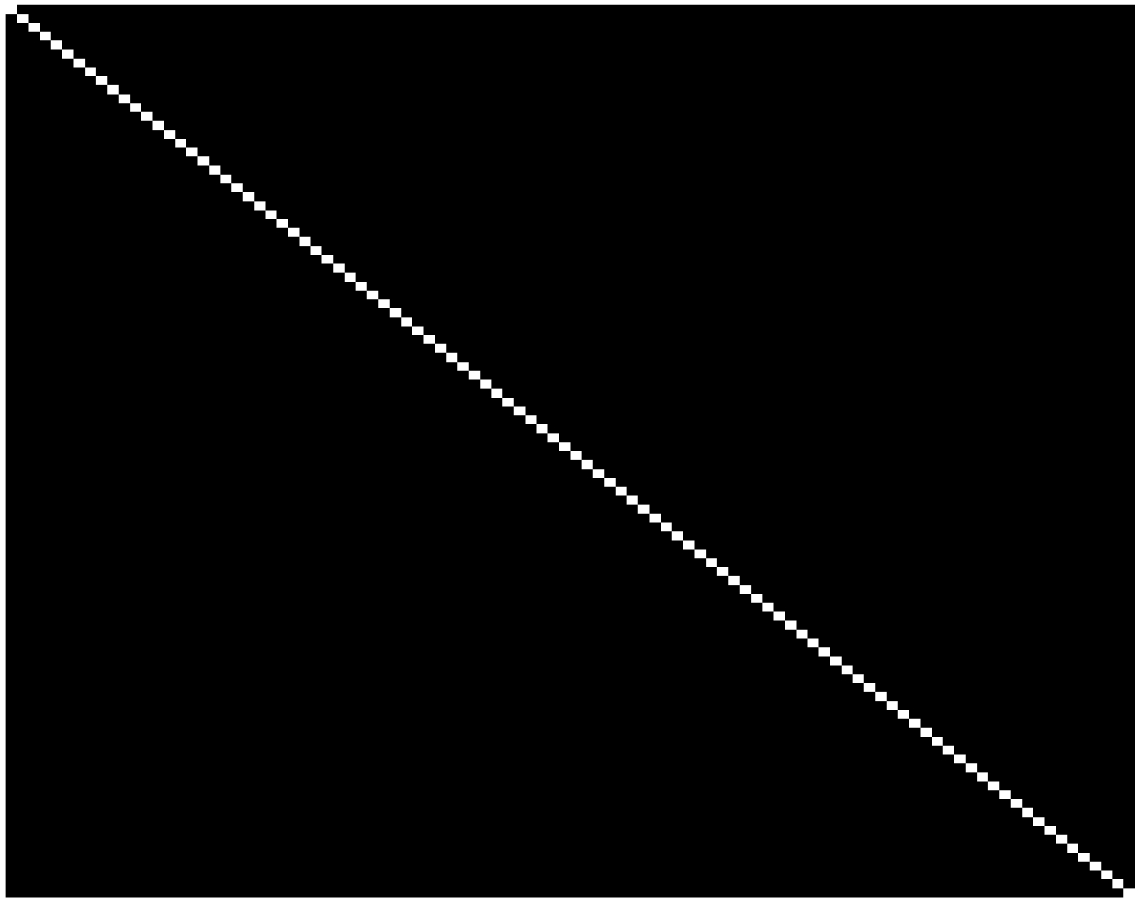}
&\includegraphics[width=0.4\textwidth,
keepaspectratio]{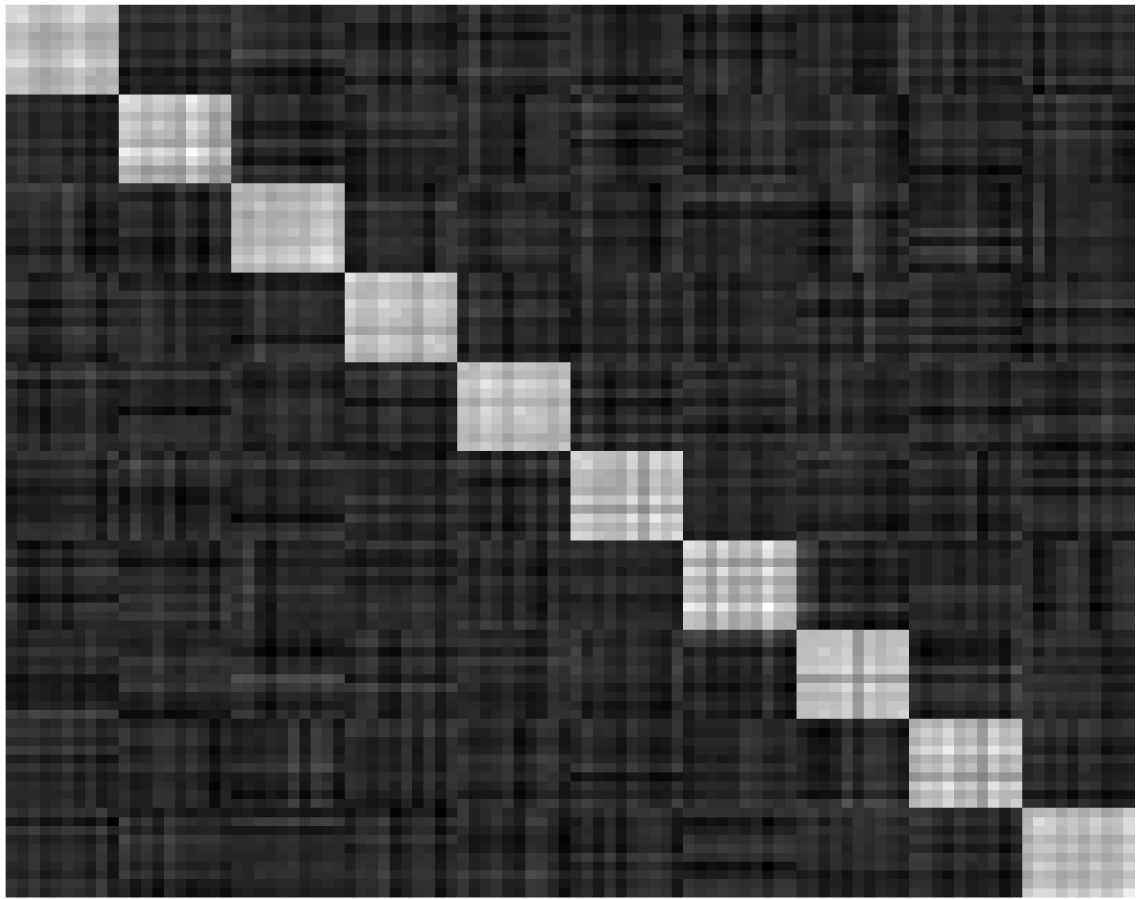}\\
(a) $\mathbf{Z}$ & (b) $\mathbf{L}\mathbf{R} $ \\
\end{tabular}
\caption{The structures of $\mathbf{Z}$ and
$\mathbf{L}\mathbf{R}$, where
$rank(\mathbf{Z})=rank(\mathbf{X})=kp=100$ and $rank(\mathbf{L}\mathbf{R})=k=10$, respectively.}\label{fig:structure}
\end{figure}

We also compared the clustering performances of $\mathbf{Z}$ and $\mathbf{L}\mathbf{R}$ on the generated data.
Fig.~\ref{fig:zvslr} shows the clustering accuracy as a function of the number of points. It can be seen that the clustering accuracy of $\mathbf{Z}$ is very sensitive to the particular sampling. Although it performs better when $p$ is increasing, the highest clustering accuracy is only around $80\%$ ($p = 30$). In contrast, $\mathbf{L}\mathbf{R}$ achieves almost perfect results on all data sets. This confirms that the affinity matrix calculated from $\mathbf{L}\mathbf{R}$ can successfully overcome the drawback of using $\mathbf{Z}$ in (\ref{eq:frr}) and LRR (also SIM) when the data sampling is insufficient.
\begin{figure}[t]
\center
\includegraphics[width=0.8\textwidth,
keepaspectratio]{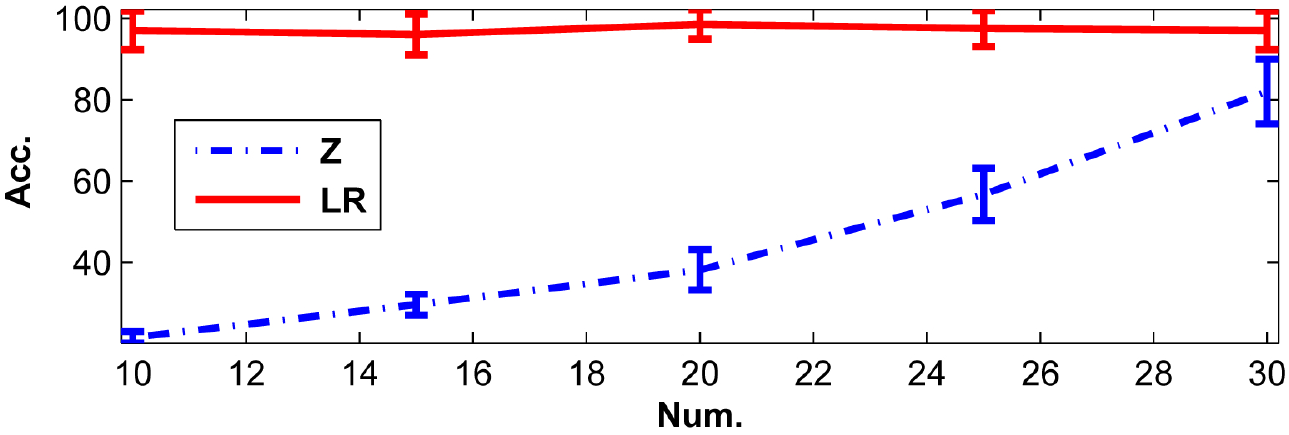}
\caption{The mean and std. clustering accuracies ($\%$) of $\mathbf{Z}$ and
$\mathbf{L}\mathbf{R}$ over 20 runs. The $x$-axis represents the number of samples in each subspace
and the $y$-axis represents the clustering accuracy.}\label{fig:zvslr}
\end{figure}

\subsubsection{Motion Segmentation}

Motion segmentation refers to the problem of segmenting tracked feature point trajectories of multiple moving objects in a video sequence. As shown in~\cite{Rao-2010-motion}, all the tracked points from a single rigid motion lie in a four-dimensional linear subspace. So this task can be regarded as a subspace clustering problem.
We perform the experiments on the Hopkins155 database~\cite{Tron-2007-Hop}, which is an extensive benchmark for motion segmentation. This database consists of 156 sequences of two or three motions thus there are 156 clustering tasks in total. For a fair comparison, we apply all algorithms
to the raw data and the parameters of these methods have been tuned to the best.

We reported the segmentation errors in Table~\ref{tab:motion} and presented the percentage of sequences for which the segmentation error is less than or equal to a given percentage of misclassification in Fig.~\ref{fig:motion}. It can be noticed that the performances of three sparsity-based models (i.e., SSC, LRR and FRR) are better than other methods.
SSC is worse than LRR because the 1D $l_1$ norm based criterion finds the representation coefficients of each vector individually, and there
is no global constrain. Although the basic forms of LRR (\ref{eq:lrr}) and FRR (\ref{eq:frr}) share the same optimal solution to $\mathbf{Z}$, FRR$_1$ performs even better than LRR in real data set. This is because enforcing the normalization constraint in (\ref{eq:rfrrn}) can improve the performance for clustering. Overall, FRR$_2$ outperforms all other methods in this paper. This result, again, confirms that $\mathbf{L}\mathbf{R}$ in FRR$_2$ is better than the general $\mathbf{Z}$ in LRR and FRR$_1$ for subspace clustering.
\begin{figure}[t]
\center
\begin{tabular}{cc}
\includegraphics[width=0.45\columnwidth]{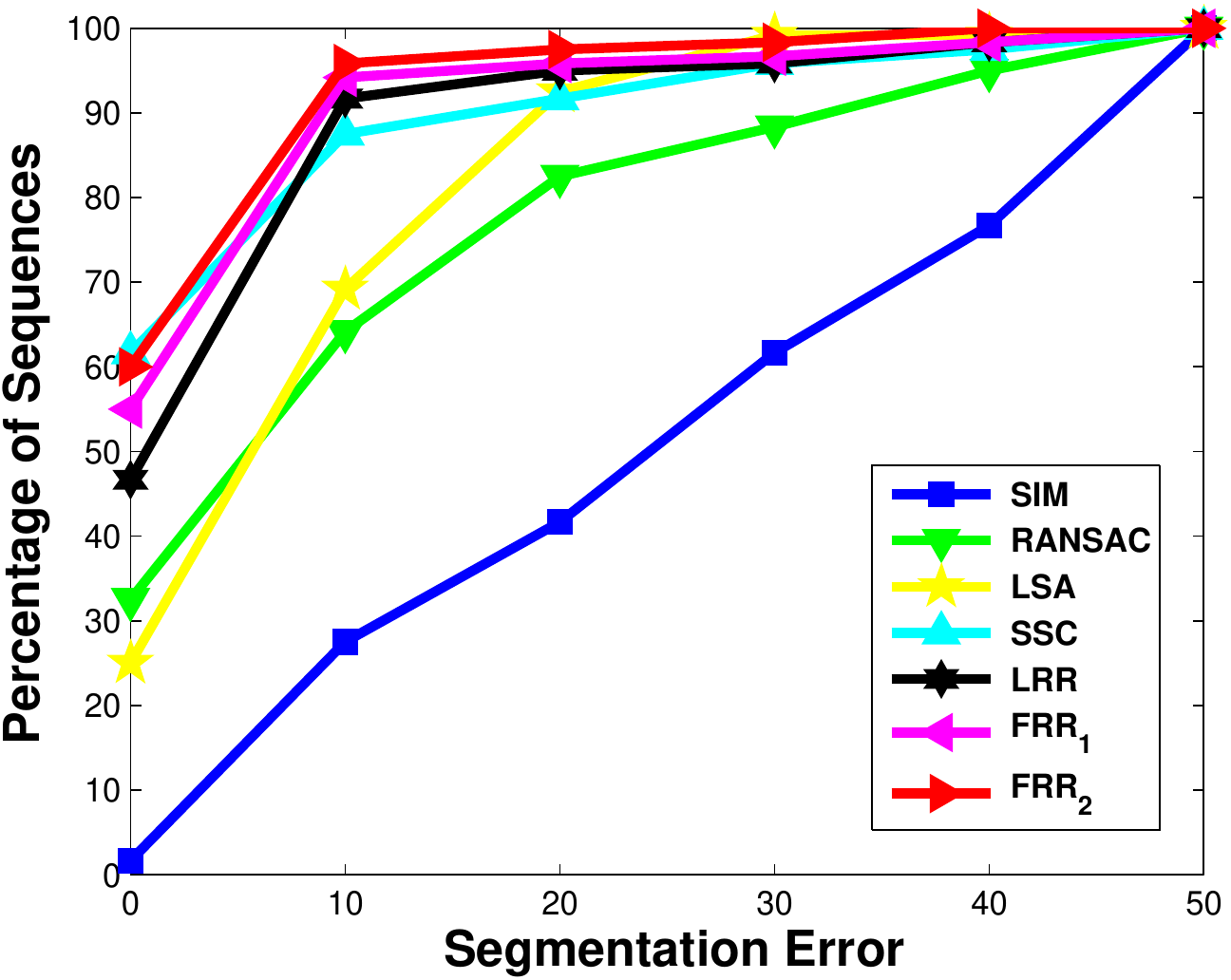}
&\includegraphics[width=0.45\columnwidth]{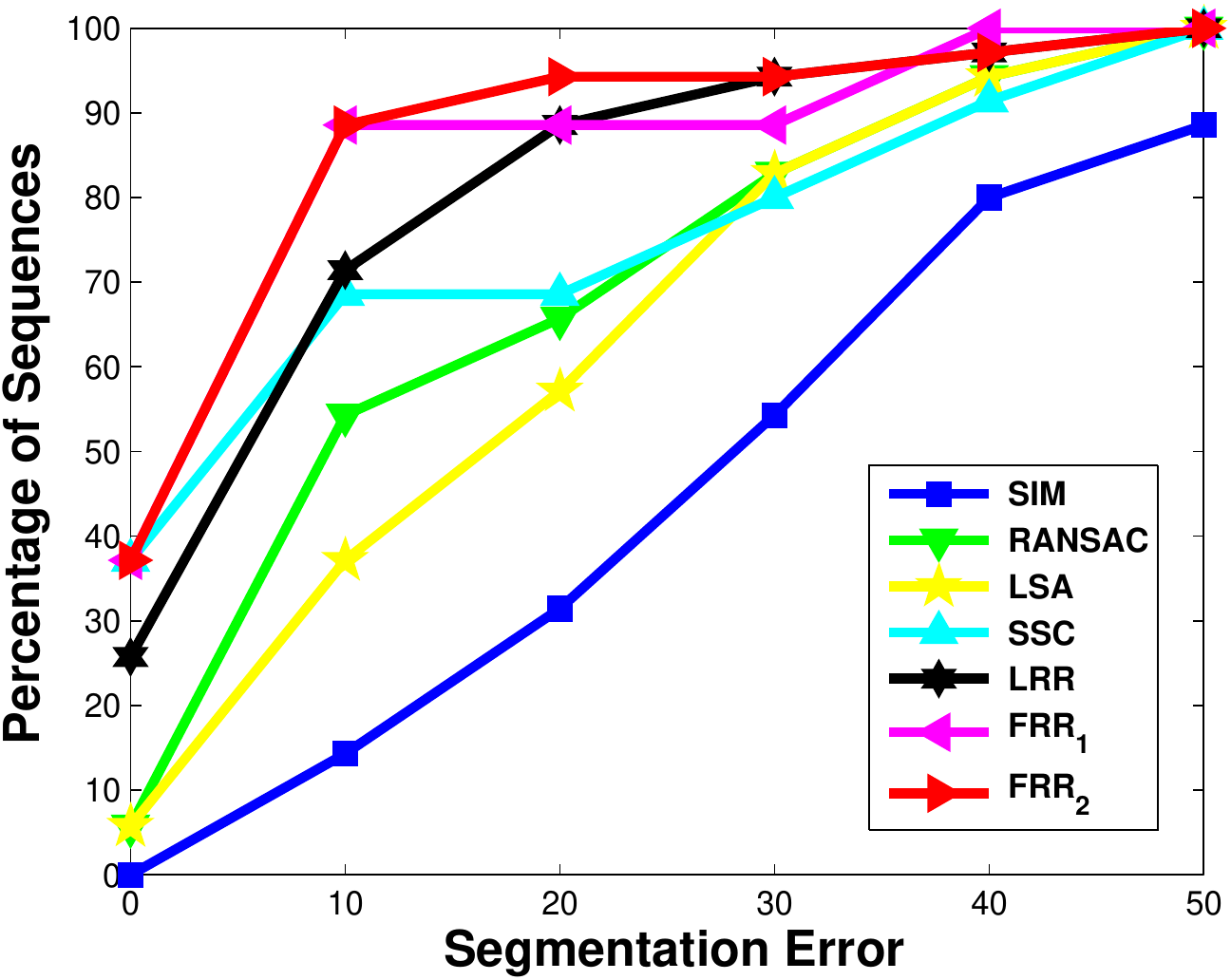}\\
(a) 2 Motions & (b) 3 Motions\\
\end{tabular}
\caption{Percentage of sequences for which the segmentation error is less than or equal to a given percentage of misclassification.}\label{fig:motion}
\end{figure}
\begin{table*}[t]
\center
\caption{Segmentation errors ($\%$) on Hopkins155 raw data.}\label{tab:motion}
\begin{tabular}{|c||c|c|c|c|c|c|c|c|c|c|c|c|}
\hline
\multirow{2}{*}{Method} & \multicolumn{4}{c|}{2 Motions} & \multicolumn{4}{c|}{3 Motions} & \multicolumn{4}{c|}{All (156)}\\\cline{2-13}
         & mean    & median & std.   & max.  & mean    & median  & std.   & max.  & mean  & median  & std.   & max. \\\hline
SIM      &  24.1   & 24.8   & 15.4   & 49.2  &  27.9   & 28.5    & 15.8   & 64.1  &  25.1 & 25.3    & 15.7   & 64.1\\
RANSAC   &  9.6    & 3.3    & 13.1   & 49.3  &  13.8   & 7.8     & 13.7   & 44.7  &  10.8 & 4.2     & 13.5   & 49.3\\
LSA      &  6.8    & 2.8    & 8.0    & 40.9  &  16.8   & 15.6    & 12.6   & 46.6  &  9.1  & 4.8     & 10.1   & 46.6\\
SSC      &  3.7    & \textbf{0.0}    & 9.7    & 49.9  &  11.4   & 3.3     & 15.0   & 44.6  &  5.5  & \textbf{0.0}     & 11.6   & 49.9\\
LRR      &  3.2    & 0.3    & 8.2    & 40.3  &  7.8    & 2.8     & 10.3   & 41.5  &  4.3  & 0.6     & 8.9    & \textbf{41.5}\\\hline
FRR$_1$ &  2.5    & \textbf{0.0}    & 7.4    & 40.8  &  5.9    & 1.4     & 10.9   & \textbf{39.4} &  3.5  & \textbf{0.0}  & 8.9    & 41.8\\
FRR$_2$ &  \textbf{1.8}  & \textbf{0.0} & \textbf{5.3} & \textbf{36.1}  &  \textbf{4.7} & \textbf{1.0} & \textbf{9.1}  & 41.5 &  \textbf{2.6} & \textbf{0.0}   & \textbf{6.5}    & \textbf{41.5}\\
\hline
\end{tabular}
\end{table*}

For three sparsity-based methods, Table~\ref{tab:moton_time} reports the time in seconds.
We can see that the computational time of SSC is lower than the standard LRR.
This is because the $l_1$ norm minimizations in SSC can be solved in parallel and there is only a thresholding process needed at each iteration. While  LRR is solved with an SVD in each iteration, and it does not scale well with large number
of samples.  By combining linearized ADM with an acceleration technique for SVD, the work in~\cite{Lin-2011-LADM} proposed a fast solver for LRR. The running time of this approach is even less than SSC.
Our FRR, again, achieves the highest efficiency because it completely avoids SVD computation in the iterations.
\begin{table}[ht]
\center
\caption{The average running time (seconds) per sequence for three sparsity-based methods. LRR(A) denotes the accelerated LRR proposed in~\cite{Lin-2011-LADM}.}\label{tab:moton_time}
\begin{tabular}{|c||c|c|c|}
\hline
Method &  2 Motions & 3 Motions & All (156)\\\hline
SSC  &  3.5445    & 7.8493    &  4.5057  \\\hline
LRR  &  38.5156   & 115.3140  &  55.6259       \\
LRR(A)  &  1.9415   & 3.6788  &  2.3319   \\\hline
FRR &  \textbf{0.9990}    & \textbf{2.2799}    &  \textbf{1.2847}       \\
\hline
\end{tabular}
\end{table}

\subsection{Feature Extraction and Outlier Detection}
This experiment tested the effectiveness of TFRR for feature extraction in presence of occlusions. To simulate sample-outliers,
we created a dataset by combining images with faces from the FRGC version 2~\cite{Philips-2005-FRGC} and
images non containing faces from Caltech-256~\cite{Griffin-2007-Caltech}. We selected $20$ images for the first $180$ subjects of the FRGC database, having a total of $3600$ images. For Caltech-256 database, which contains 257 image categories, we randomly selected 1 image from each class (a total of 257 non-facial images). All images are resized to $32 \times 36$ and the pixel values are normalized to $[0, 1]$. As shown in Fig.~\ref{fig:face},
there are two types of corruptions: small errors in the facial images (e.g., illuminations and occlusions) and non-facial outliers.
\begin{figure}[t]
\center
\includegraphics[width=0.8\columnwidth]{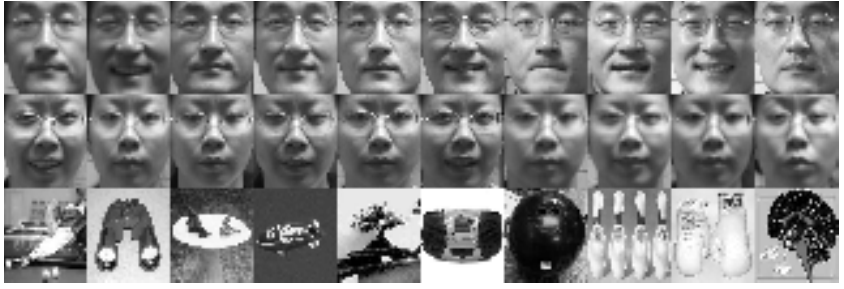}
\caption{Examples of the FRGC-Caltech data set.
The top two rows correspond to face images and the bottom
row non-face images.}\label{fig:face}
\end{figure}

The goal of this task is to robustly extract facial features and use them for classification.
 That is, we learn a mapping $\mathbf{P}$ between high dimensional observations and low dimensional features
using TFRR, and identify outliers in the training set by $\mathbf{E}$.
Then for a new testing data $\mathbf{x}$, the feature vector $\mathbf{y}$ can be computed as $\mathbf{y}=\mathbf{P}\mathbf{x}$. We selected
the first $k$ ($k = 40, 80$) identities and $257$ non-facial images as the training set and the remaining $(180-k)$ identities of facial images for test.
We compared two TFRR based strategies (one is directly using $\mathbf{P}=\mathbf{Z}$, called TFRR$_1$, and another is computing
the orthogonal basis $\mathbf{P}=\mbox{orth}(\mathbf{L}\mathbf{R})$, called TFRR$_2$) with the ``Raw data'' baseline and other state-of-the-art approaches, such as PCA, Locality Preserving Projection (LPP) \cite{He-2003-LPP} and Neighborhood Preserving Embedding (NPE) \cite{He-2005-NPE}. The parameters and the feature dimensions of all methods are tuned to the best for each training set. Table~\ref{tab:face} demonstrates that the performances of TFRR$_1$ and TFRR$_2$ are both significantly better than the baseline and PCA. Moreover, TFRR$_2$ outperforms all other methods on these experiments.

As shown in Fig.~\ref{fig:XZXE}, the main advantage of TFRR based methods comes from their ability of extracting intrinsic facial features and removing outliers. One can see that most of the intrinsic facial features can be projected into the range space (modeled by $\mathbf{Z}\mathbf{X}$, see the middle row), while the small errors of the facial images (e.g., illuminations and occlusions) and non-facial outliers (modeled by $\mathbf{E}$) can be automatically removed (see the bottom row).
\begin{table*}[ht]
\center
\caption{Classification accuracies (mean $\pm$ std.$\%$) on FRGC-Caltech data set. ``G$m$/P$n$'' means in the testing data $m$ images of each subject
 are randomly selected as gallery set
and the remaining $n$ images as probe set. Such a trial is repeated 20 times. The feature dimensions
are: PCA (410D, 358D), LPP (170D, 200D), NPE(320D, 160D) and TFRR$_2$ (190D, 100D). The dimension of the feature vector produced by TFRR$_1$ is the same as the observed data.}\label{tab:motion}
\begin{tabular}{|c|c||c|c|c|c|c|c|}
\hline
      Train                &   Test & Raw             & PCA            & LPP            & NPE            & TFRR$_1$    & TFRR$_2$    \\
\hline
\multirow{2}{*}{$40\times 20 + 257$} & G5/P15    & 71.1 $\pm$ 3.2  & 70.0 $\pm$ 3.2 & 85.2 $\pm$ 2.4 & 81.1 $\pm$ 2.7 & 81.5 $\pm$ 2.0 & \textbf{88.8 $\pm$ 2.7}\\
                                     & G10/P10& 82.8 $\pm$ 4.6  & 81.6 $\pm$ 4.6 & 92.2 $\pm$ 2.8 & 89.6 $\pm$ 3.6 & 89.9 $\pm$ 2.7 & \textbf{94.1 $\pm$ 2.1}\\\hline
\multirow{2}{*}{$80\times 20 + 257$} & G5/P15    & 72.3 $\pm$ 4.1  & 71.4 $\pm$ 4.1 & 85.4 $\pm$ 2.9 & 83.7 $\pm$ 4.2 & 82.9 $\pm$ 3.3 & \textbf{90.8 $\pm$ 2.1}\\
                                     & G10/P10& 82.6 $\pm$ 3.2  & 81.6 $\pm$ 3.2 & 91.4 $\pm$ 3.2 & 90.4 $\pm$ 3.2 & 90.1 $\pm$ 2.1 & \textbf{94.9 $\pm$ 2.9}\\
\hline
\end{tabular}\label{tab:face}
\end{table*}
\begin{figure}[t]
\center
\begin{tabular}{rc@{\extracolsep{0.2em}}cc@{\extracolsep{0.2em}}c}
\multicolumn{1}{r}{$\mathbf{X}$:}
&\includegraphics[width=0.1\textwidth,
keepaspectratio]{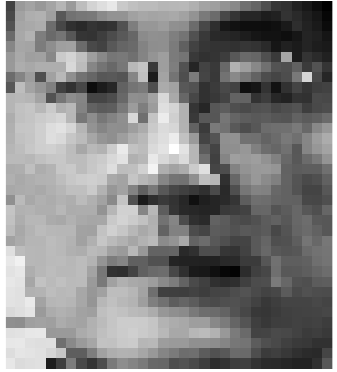}
&\includegraphics[width=0.1\textwidth,
keepaspectratio]{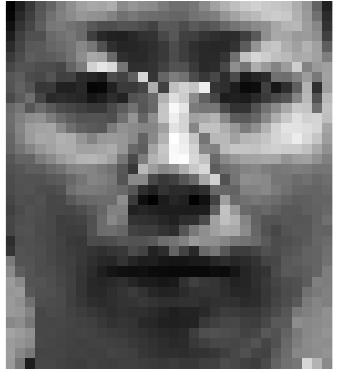}
&\includegraphics[width=0.1\textwidth,
keepaspectratio]{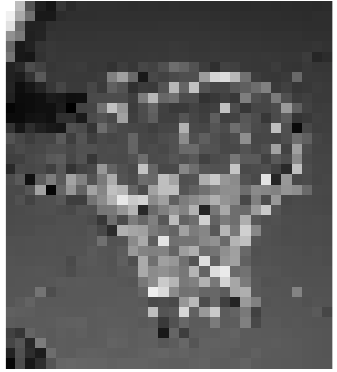}
&\includegraphics[width=0.1\textwidth,
keepaspectratio]{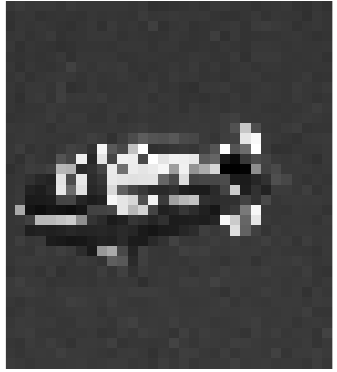}\\
$\mathbf{Z}\mathbf{X}$:
&\includegraphics[width=0.1\textwidth,
keepaspectratio]{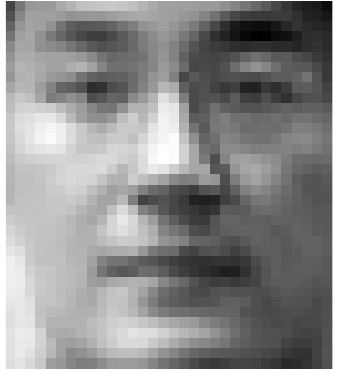}
&\includegraphics[width=0.1\textwidth,
keepaspectratio]{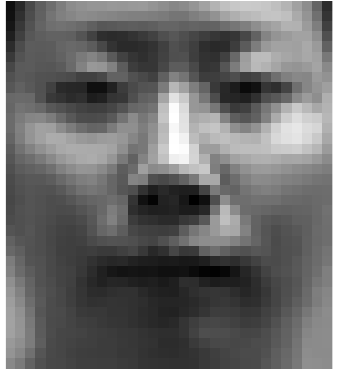}
&\includegraphics[width=0.1\textwidth,
keepaspectratio]{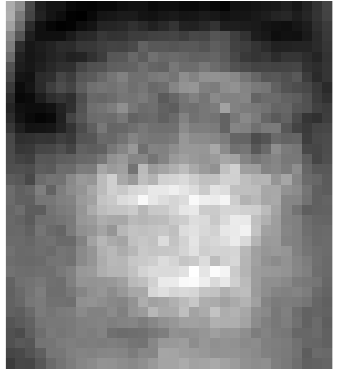}
&\includegraphics[width=0.1\textwidth,
keepaspectratio]{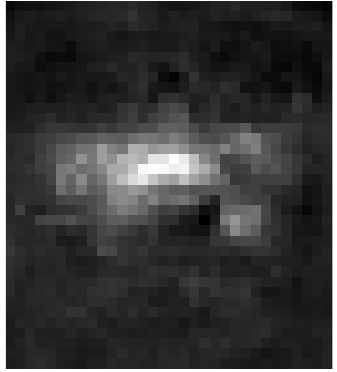}\\
$\mathbf{E}$:
&\includegraphics[width=0.1\textwidth,
keepaspectratio]{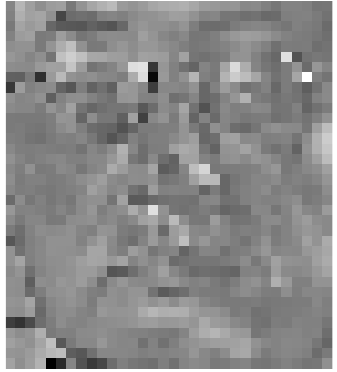}
&\includegraphics[width=0.1\textwidth,
keepaspectratio]{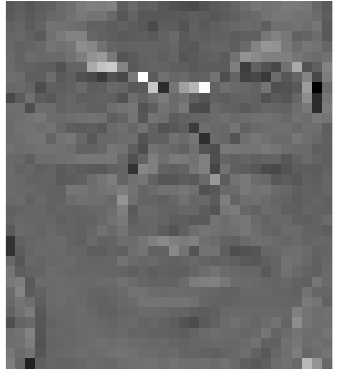}
&\includegraphics[width=0.1\textwidth,
keepaspectratio]{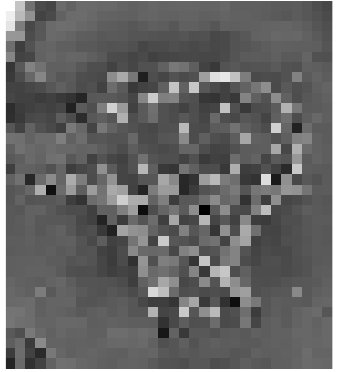}
&\includegraphics[width=0.1\textwidth,
keepaspectratio]{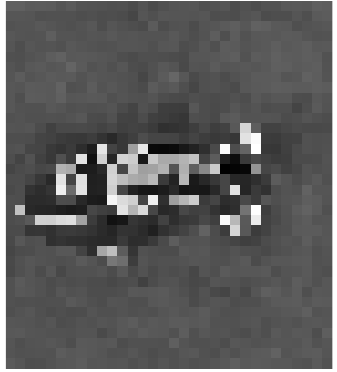}\\
\end{tabular}
\caption{Some examples of using TFRR to recover the intrinsic facial features and remove small errors and outliers
(modeled by $\mathbf{X}=\mathbf{Z}\mathbf{X} + \mathbf{E}$). The left two columns correspond to facial samples and the right two are non-facial samples. The middle row shows the features extracted by our algorithm $(\mathbf{Z}\mathbf{X})$ and the bottom row shows the corruptions ($\mathbf{E}$).}\label{fig:XZXE}
\end{figure}

Fig.~\ref{fig:outlier} plotted the energies (in terms of $l_2$ norm) for the columns of $\mathbf{E}$. One can see that the values of non-facial samples (last 257 columns in $\mathbf{E}$) are obviously larger than
that of facial samples. Therefore, the error term $\mathbf{E}$ can also be used to detect the non-facial outliers. Namely the $i$-th sample in $\mathbf{X}$ is considered as outlier if and only if $\|[\mathbf{E}]_i\|_{2} \geq \gamma$. By setting the parameter $\gamma = 2.2$, the outlier detection accuracies\footnote{These
accuracies are obtained by computing the percentage of correctly identified outliers. One may also consider the receiver operator characteristic (ROC) and
compute its area under curve (AUC) \cite{Liu-2011-LRR} to evaluate the performance.} are $98.68\%$ on the $40\times 20 + 257$ data and $99.19\%$ on $80\times 20 + 257$ data, respectively.
\begin{figure}[t]
\center
\begin{tabular}{cc}
\includegraphics[width=0.45\columnwidth]{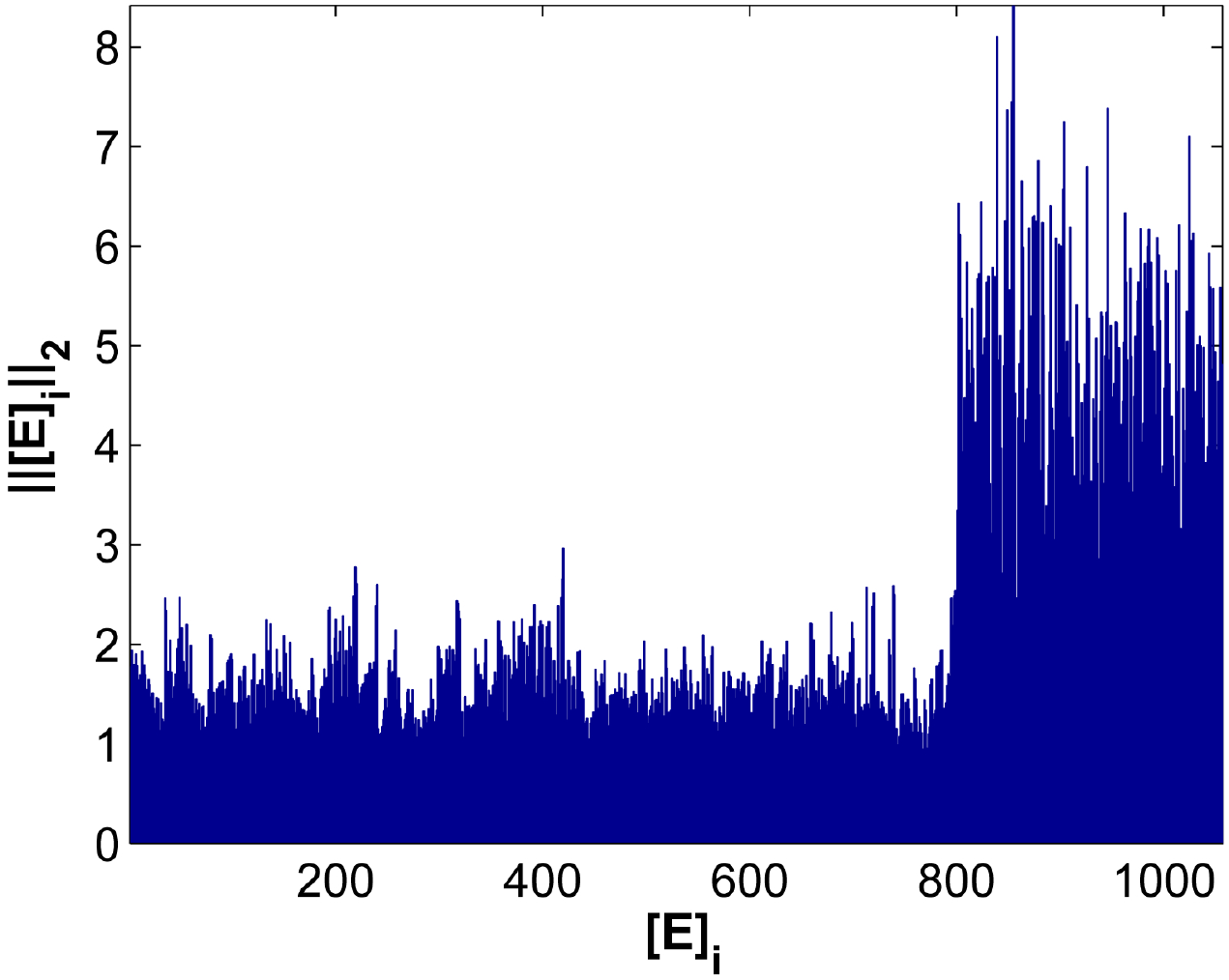}
&\includegraphics[width=0.45\columnwidth]{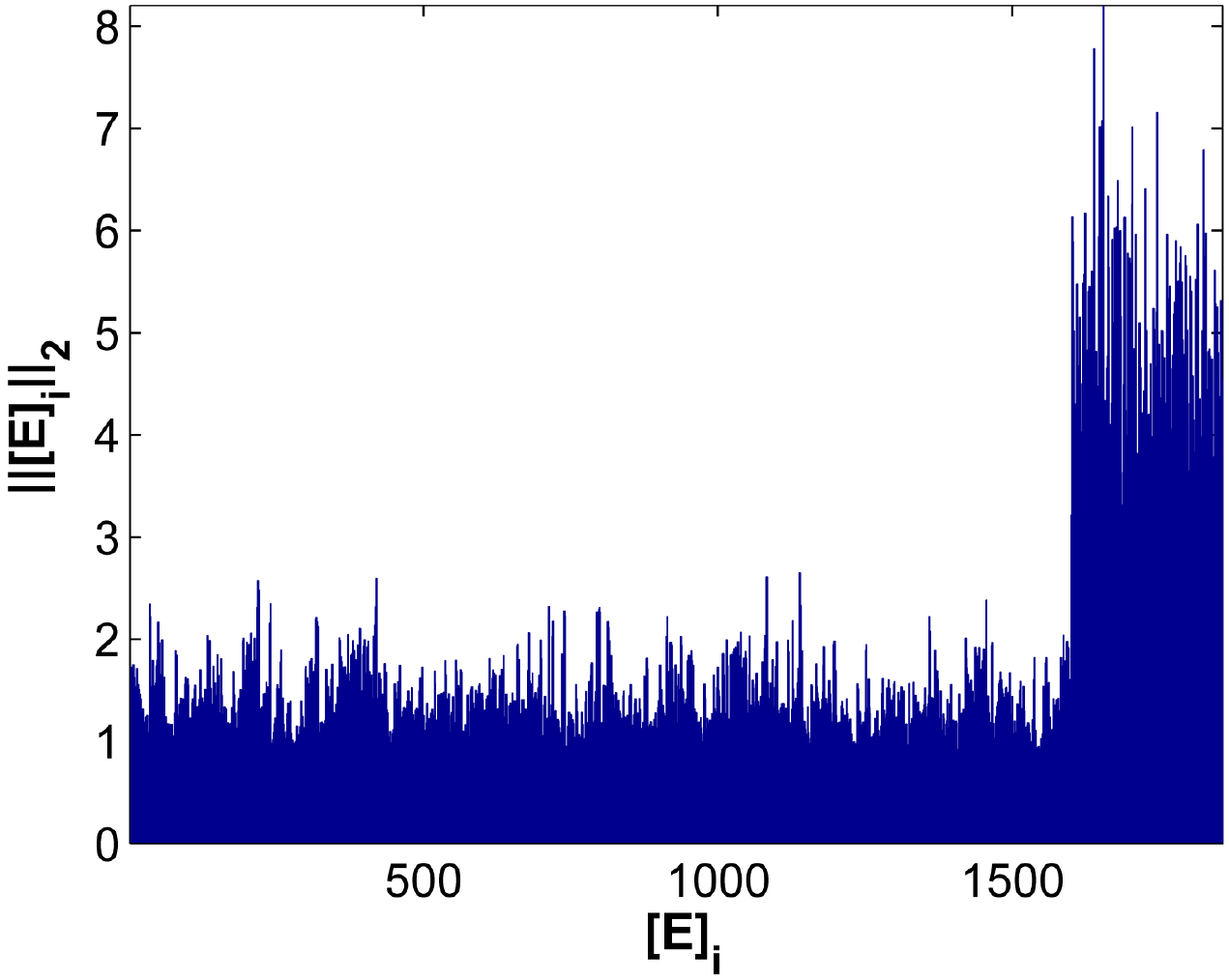}\\
(a) $40\times 20 + 257$ & (b) $80\times 20 + 257$\\
\end{tabular}
\caption{The $l_2$ norm for the columns of $\mathbf{E}$. The first 800 (a) and 1600 (b) columns are facial images
and the last 257 columns are outliers.}\label{fig:outlier}
\end{figure}

\section{Conclusions}\label{sec:con}
This paper proposed a novel framework, named fixed-rank representation (FRR), for robust unsupervised visual learning.
We proved that FRR can reveal the multiple subspace structure for clustering, even with insufficient observations.
 We also demonstrated that the transposed FRR (TFRR) can successfully recover the column space, and thus can be applied for feature extraction. There remain several directions for future work:
1) provide a deeper analysis on $\mathbf{L}\mathbf{R}$ (e.g., the general strategy for choosing efficient basis from $\mathcal{R}(\mathbf{Z})$ for subspace clustering and determining dimension for feature extraction), 2) apply FRR to supervised and semi-supervised learning.

\ifCLASSOPTIONcompsoc
  \section*{Acknowledgments}
\else
  \section*{Acknowledgment}
\fi

This work is supported by the NSFC-Guangdong Joint Fund (No.U0935004), the NSFC Fund (No.61173103)
and the Fundamental Research Funds for the Central Universities.
R. Liu would also like to thank the support from CSC.

\ifCLASSOPTIONcaptionsoff
  \newpage
\fi



%
\bibliographystyle{IEEEtran}
\bibliography{frr}
\end{document}